\documentclass{article}

\usepackage{times}
\usepackage{graphicx} 
\usepackage{subfigure} 

\usepackage[square,sort,comma,numbers]{natbib}

\usepackage{algorithm}
\usepackage{algorithmic}

\usepackage{hyperref}

\usepackage{helvet}
\usepackage{courier}

\usepackage{makecell}
\usepackage{graphicx}
\usepackage{subfigure}
\usepackage{color}
\usepackage{amsfonts}
\usepackage{amsmath}
\usepackage{amssymb}

\def\etal{{\em et al.}}

\def\mD{{\mathcal D}}

\def\mH{{\mathcal H}}

\def\mP{{\mathcal P}}

\def\mS{{\mathcal S}}

\def\mX{{\mathcal X}}
\def\mY{{\mathcal Y}}
\def\mZ{{\mathcal{Z}}}

\DeclareMathAlphabet\mathbfcal{OMS}{cmsy}{b}{n}

\def\0{{\bf 0}}
\def\1{{\bf 1}}

\def\bG{{\bf G}}

\def\bI{{\bf I}}

\def\bM{{\bf M}}

\def\bx{{\bf x}}
\def\by{{\bf y}}
\def\bz{{\bf z}}

\def\mmE{{\mathbb E}}

\def\bx{{\bf x}}

\def\by{{\bf y}}

\def\bz{{\bf z}}

\newtheorem{deftn}{Definition}
\newtheorem{thm}{Theorem}

\newtheorem{remark}{Remark}

\newtheorem{proof}{Proof}

\def\blue{\textcolor{blue}}

\usepackage[preprint]{nips_2018}

\usepackage[utf8]{inputenc} 
\usepackage[T1]{fontenc}    
\usepackage{url}            
\usepackage{booktabs}       
\usepackage{amsfonts}       
\usepackage{nicefrac}       
\usepackage{microtype}      

\usepackage{multirow}%
\usepackage{epstopdf}

\def\zhang{\textcolor{black}}
\def\kui{\textcolor{black}}

\def\ygyan{\textcolor{black}}
\def\qi{\textcolor{black}}
\def\yong{\textcolor{black}}

\def\error{\textcolor{black}}

\title{Dual Reconstruction Nets for Image Super-Resolution with Gradient Sensitive Loss}

\author{
	Yong Guo$^1$, Qi Chen$^1$, Jian Chen$^1$, Junzhou Huang$^2$,\\
	\textbf{Yanwu Xu$^3$, Jiezhang Cao$^1$, Peilin Zhao$^4$, Mingkui Tan$^1$\thanks{Corresponding author.}} \\
	$^1$South China University of Technology, $^2$University of Texas at Arlington, \\
	$^3$Guangzhou Shiyuan Electronics Co.,Ltd, $^4$Tencent AI Lab \\
	\{guo.yong, sechenqi, secaojiezhang\}@mail.scut.edu.cn, \{ellachen, mingkuitan\}@scut.edu.cn, \\
	jzhuang@uta.edu, xuyanwu@cvte.com, peilinzhao@hotmail.com
}

\begin{document}

\maketitle

\begin{abstract}
Deep neural networks have exhibited promising performance in image super-resolution \yong{(SR)} due to the power in learning the non-linear mapping from low-resolution (LR) images to high-resolution (HR) images. However, most deep learning methods employ feed-forward architectures, \yong{and} thus the dependencies between LR and HR images are not fully exploited, leading to limited learning performance. Moreover, most deep learning based SR \yong{methods} apply the pixel-wise reconstruction error as the loss, which, however, may fail to capture high-frequency information and \yong{produce} perceptually unsatisfying results, whilst the recent perceptual loss relys on some pre-trained deep model and they may not generalize well. In this paper, we introduce a mask to separate the image into low- and high-frequency parts based on image gradient magnitude, and then devise a gradient sensitive loss to well capture the structures in the image without sacrificing the recovery of low-frequency \yong{content}. Moreover, by investigating the duality in SR, we develop a dual reconstruction network (DRN) to improve the SR performance. We provide theoretical analysis on the generalization performance of our method  and demonstrate its effectiveness and superiority with thorough experiments.
\end{abstract}

\section{Introduction}

Super-resolution (SR) aims to learn a nonlinear mapping to reconstruct high-resolution (HR) images from low-resolution (LR) input images, and it has been widely desired in many real-world scenarios, including image/video reconstruction~\cite{dong2014learning, huang2015single, kim2016deeply, shi2016real}, fluorescence microscopy~\cite{nehme2018deep} and face recognition~\cite{chowdhuri2012very}. SR is a \yong{typical} ill-posed inverse problem. In the last two decades, many attempts have been made to address it, mainly including interpolation based methods~\cite{yang2014single} and reconstruction based methods~\cite{dong2014learning, gu2015convolutional,jianchao2008image,kim2016accurate, huang2015single,ledig2016photo}.

Recently, deep neural networks (DNNs) have emerged as a powerful tool for SR~\cite{dong2014learning} and have shown significant advantages over traditional methods in terms of performance and inference speed~\cite{dong2014learning, ledig2016photo, lai2017deep, lim2017enhanced}. However, these methods may have some underlying limitations. First, \yong{for} deep learning based SR methods, the performance highly depends on the choice of the loss function~\cite{ledig2016photo}. The most widely applied loss is the pixel-wise error  between the recovered HR image and the ground truth image, such as the mean squared error (MSE) and mean absolute error (MAE). This kind of losses is helpful to improve the peak signal-to-noise (PSNR), a common measure to evaluate SR algorithms. However, they may make the model lose the high-frequency details and thus fail to capture perceptually relevant differences~\cite{johnson2016perceptual, ledig2016photo}. To address this issue, recently, some researchers have developed the perceptual loss~\cite{johnson2016perceptual,ledig2016photo} to produce photo-realistic images. However, the computation of perceptual loss relies on some pre-trained model (such as VGG~\cite{simonyan2014very}), which has profound influence on the performance. In particular, the method may not generalize well if the pre-trained model is not well trained. \yong{In practice, it may incur significant changes in \yong{image} content (see Figure 2 in~\cite{ledig2016photo} and Figure \ref{fig:image_compare_4x} in this paper).
}


\kui{Moreover, most deep learning methods are trained in a simple \yong{feed-forward} scheme and do not fully exploit the mutual dependencies between low- and high-resolution images~\cite{DBPN2018}. To improve the performance, one may increase the depth or width of the networks, which, however, may \yong{incur} more memory consumption and computation cost, and require more data \yong{for training.}
	To address this issue, the back-projection has been investigated~\cite{DBPN2018}. Specifically, a deep back-projection network (DBPN) is developed to \yong{improve the learning performance.} However, in DBPN, the dependencies between LR and HR images are still not fully exploited, since it does not consider the loss between the down-sampled image and the original LR image. As a result, the representation capacity of deep models may not be well exploited.}


We seek to address the above issues in two directions. \yong{\textbf{First}}, we devise a novel gradient sensitive loss relying on the gradient magnitude of an image. To do so, we hope to well recover both low-  and  high-frequency information at the same time. To achieve high performance of SR, besides the pixel-wise PSNR, we also seek to achieve high PSNR \yong{score} over image gradients. \yong{\textbf{Second}}, by exploiting the duality in LR and SR images, we formulate the SR problem as a dual learning task and we present a dual reconstruction network (DRN) by introducing an additional dual module to exploit the bi-directional information of LR and HR images.

In this paper, we make the following contributions. First, we devise a novel gradient sensitive loss to improve the reconstruction performance. Second, we develop a dual learning scheme for image super-resolution by exploiting the mutual dependencies between low- and high-resolution images via the  task duality. Third, we theoretically  prove the effectiveness of the proposed dual reconstruction scheme for SR in terms of generalization ability. \kui{Our result on generalization bound of dual learning is more general than \citep{pmlr-v70-xia17a}.} Last, extensive experiments demonstrate the effectiveness of the proposed gradient sensitive loss and dual reconstruction scheme.

\section{Related work}

\textbf{Super-resolution.}
\zhang{One classic SR method is the interpolation-based approach, such as cubic-based~\cite{hou1978cubic}, edge-directed~\cite{allebach1996edge,li2001new} and wavelet-based~\cite{nguyen2000efficient} methods. These methods, however, may oversimplify the SR problem and usually generate blurry images with overly smooth textures~\cite{ledig2016photo,tong2017image}.}
Besides, there are some other methods, such as sparsity-based techniques~\cite{gu2015convolutional,jianchao2008image} and neighborhood embedding~\cite{gao2012image,timofte2013anchored}, which have been widely used in real-world applications.

\zhang{Another classic method is the reconstruction-based method~\cite{baker2002limits,ben2007penrose,lin2004fundamental}, which takes LR images to reconstruct the corresponding HR images. Following such method, many CNN-based methods~\cite{hui2018fast,kim2016accurate,NIPS2016_6172, tai2017image,tong2017image,wang2018recovering,zhang2017learning,zhang2018residual} were developed to learn a reconstruction mapping and achieve state-of-the-art performance. However, all these methods only consider the information from HR images and ignore the mutual dependencies between LR and HR images. Very recently, Haris~\etal\cite{DBPN2018} propose a back-projection network and find that mutual dependencies are able to enhance the performance of SR algorithms. 
}


\textbf{Loss function.}
The loss function plays a very important role in image super-resolution. The mean squared error (MSE)~\cite{dong2014learning, huang2015single, kim2016accurate} and mean absolute error (MAE)~\cite{zhang2018residual} are two widely used loss functions:
\begin{equation}
\begin{array}{ll}
\ell_{\mathrm{MSE}}(\bI^{\mathrm{H}}, \hat \bI^{\mathrm{H}}) = \left\| \bI^{\mathrm{H}} - \hat \bI^{\mathrm{H}} \right\|_F^{2}, \mathrm{~and~~~}  \ell_{\mathrm{MAE}}(\bI, \widehat{\bI}) = \left\| \bI - \hat \bI\right\|_1,
\end{array}
\end{equation}
where $\left\| \cdot \right\|_1$ denotes $\ell_1$-norm, and $\bI$ and $\widehat \bI$ denote the ground-truth and the predicted images.
While $\ell_{\mathrm{MSE}}$ is a standard choice which is directly related to PSNR~\cite{dong2014learning}, $\ell_{\mathrm{MAE}}$ may be a better choice to produce sharp results~\cite{lim2017enhanced}. Nevertheless, the two loss functions can be used simultaneously. For example, in \cite{hui2018fast} they first train the network with $\ell_{\mathrm{MAE}}$ and
then fine-tune it by $\ell_{\mathrm{MSE}}$.
Recently, Lai~\etal\cite{lai2017deep} and Liao~\etal\cite{liao2015video} introduce the Charbonnier penalty which is a variant of $\ell_{\mathrm{MAE}}$.
Justin~\etal\cite{johnson2016perceptual} propose a perceptual loss, by minimizing the reconstruction error based on the extracted features, to improve the perceptual quality. More recently, Ledig~\etal\cite{ledig2016photo} leverages the adversarial loss to produce photo-realistic images. 
However, these methods\footnote{The summarization and comparison of loss functions can be found in Table~\ref{tab:loss_compare} of supplementary file.} take the whole image as input and do not distinguish between low- and high-frequency details. As a result, the low-frequency content and high-frequency structure information cannot be fully exploited.

\section{Gradient sensitive loss}
In this section, we propose a {gradient-sensitive} loss in order to preserve both low-frequency content and high-frequency structure of images for image super-resolution.
\error{As aforementioned}, optimizing the pixel-wise loss often lacks high-frequency structure information and may produces perceptually blurry images. To address this issue, we hope to  recover the image gradients as well in order to capture the high frequency structure information. Intuitively, one may exploit  the loss over gradients~\cite{mathieu2015deep}:
\begin{equation}
\begin{array}{ll}
\ell_{\mathrm{G}}(\bI, \widehat \bI) = \left\| \nabla_x \bI - \nabla_x \widehat \bI \right\|_1
+ \left\| \nabla_y \bI - \nabla_y \widehat \bI \right\|_1,
\end{array}
\end{equation}
where $\nabla_x \bI$ and $\nabla_y \bI$ denote the directional gradients of $\bI$ along the horizontal (denoted by $x$) and vertical (denoted by  $y$) directions, respectively.
Apparently, we cannot directly minimize $\ell_{\mathrm{G}}$ for SR. Instead, we can construct a joint loss by considering both pixel-level and gradient-level errors:
\begin{equation}\label{eq:pixel_gradient_loss}
\begin{array}{ll}
\ell_{\mathrm{GP}}(\bI, \widehat \bI) =
\ell_{\mathrm{G}} \left(\bI,\widehat \bI \right)
+ \lambda \ell_{\mathrm{P}} \left(\bI,\widehat \bI \right),
\end{array}
\end{equation}
where $\ell_{\mathrm{P}} $ is the pixel-level loss, which can be either $\ell_{\mathrm{MSE}} $ or $\ell_{\mathrm{MAE}}$ and $\lambda$ is a parameter to balance the two terms. Minimizing $\ell_{\mathrm{GP}}$ in (\ref{eq:pixel_gradient_loss}) will help to recover the gradients, but it means we have to sacrifice the accuracy over the pixels, and a good balance is often hard to made. In other words, the reconstruction performance in terms of PSNR over the image shall degrade when considering the recovery of gradients.

A natural questions arises: given an image, can we find a way to separate the high-frequency part from its low-frequency part and then impose losses over the two parts separately? If the answer is positive, the emphasis on the gradient-level loss will not affect the pixel-level loss and then the dilemma within $\ell_{\mathrm{GP}}$ shall be addressed.
Here, we develop a simple method concerning the above question. Specifically, we seek to find a mask $\bM$ to decompose the image $\bI$ by $$\bI = \bM  \odot \bI + (\1 -\bM)   \odot \bI,$$ where $M_{i, j} \in [0, 1]$.  Relying on the directional gradients  $\nabla_x \bI^{\mathrm{H}}$ and $\nabla_y \bI^{\mathrm{H}}$, we can easily devise such a mask. In fact, given the gradient magnitude $\bG$, where $ G_{i, j}= \sqrt{ (\nabla_x I_{i, j})^2 + (\nabla_y I_{i, j})^2 },$ we can define the mask as the normalization of $\bG$ into $[0, 1]$:
\begin{equation}
\bM = ({\bG-\min(\bG)})/({\max(\bG)-\min(\bG)}),
\end{equation}
where $\min(\bG)$ and $\max(\bG)$ denote the minimum and maximum value in $\bG$, respectively. It is clear that $\bM  \odot \bI$ and $(\1 -\bM)\odot\bI$ represent the low and high-frequency parts, separately.
Finally, we define our \textbf{gradient-sensitive} loss as
\begin{equation}\label{eq:gradient_masked_loss}
\ell_{\mathrm{GS}}(\bI, \hat \bI) =
\ell_{\mathrm{G}} (\bM \odot \bI, \bM \odot \widehat \bI )
+ \lambda \ell_{\mathrm{P}} ((\1 -\bM) \odot \bI, (\1 -\bM) \odot \widehat \bI ),
\end{equation}
where $\odot$ denotes the element-wise multiplication and $\lambda$ is a trade-off parameter.  Here, we adopt  $\ell_{\mathrm{MAE}}$ as the pixel-level loss $\ell_{\mathrm{P}} $.

In Eqn. (\ref{eq:gradient_masked_loss}), the gradient-level loss $\ell_{\mathrm{G}} (\bM \odot \bI, \bM \odot \widehat \bI )$ focuses on the high-frequency part and it helps to improve the reconstruction accuracy of gradients. Different from $\ell_{\mathrm{GP}}$ in (\ref{eq:pixel_gradient_loss}), in $\ell_{\mathrm{GS}}$, the pixel-level loss $\ell_{\mathrm{P}}$ focuses on low-frequency part, since the gradient information has been subtracted from $\bI$. As a result, the reconstruction accuracy over pixels will not suffer even though we put emphasis on gradient. Most importantly, since the gradient information can be well recovered, it will help to improve the overall performance in terms of PSNR and
\yong{visual quality} significantly.\footnote{We conduct an experiment to demonstrate the effectiveness of gradient sensitive loss (See Section \ref{Ablation_Study}).}

\section{Dual reconstruction network}

Most existing methods employ feed-forward architecture and focus on minimizing the reconstruction error between recovered image and the ground-truth. As such, they ignore the mutual dependencies between LR and HR images. As a result, the representation capacity are not fully exploited~\cite{DBPN2018,xia2017dual}. Here, we seek to investigate the duality in SR problems and propose a dual reconstruction scheme to fully exploiting the mutual dependencies  between LR and HR images to improve the performance.

\subsection{Dual reconstruction scheme for super-resolution}
Dual supervised learning (DSL) has been investigated in~\cite{xia2017dual, he2016dual} and shows that DSL can improve the practical performances of both tasks. Inspired by~\cite{xia2017dual}, we introduce an additional dual reconstruction of LR to improve the primal reconstruction of HR images.
We aim to simultaneously learn the primal mapping $P(\cdot)$ to reconstruct HR images and the dual mapping $D(\cdot)$ to reconstruct LR images.
Let $\bx \in \mathcal{X}$ be LR images and $\by \in \mathcal{Y}$ be HR images.
We formulate the SR problem as a dual reconstruction learning scheme as below. 
\begin{deftn} \textbf{\emph{(Primal learning task) }}\label{definition: primal_model}
	The primal learning task aims to find a function $P$: $\mX \rightarrow \mY$, such that the prediction $P(\bx)$ is similar to its corresponding HR image $\by$.
\end{deftn}
\begin{deftn} \textbf{\emph{(Dual learning task) }}\label{definition: dual_model}
	The dual learning task aims to find a function $P$: $\mY \rightarrow \mX$, such that the prediction of $D(\by)$ is similar to the original input LR image $\bx$.
\end{deftn}
If $P(\bx)$ were the correct HR image, then the down-sampled images $D(P(\bx))$ should be very close to the input LR images $\bx$.
In other words, the dual reconstruction of LR images is able to provide additional supervision to learn a better primal reconstruction mapping.
\error{To train the proposed model}, we construct a dual reconstruction loss which can be computed as follows:
\begin{equation}
\label{eq:dual_reconstruction}
\mathcal{L}_{\mathrm{DR}} (\bx, \by) = \ell_1 \Big( P(\bx), \by \Big) + \ell_2 \Big( D(P(\bx)), \bx \Big),
\end{equation}
where $\ell_1(\cdot)$ and $\ell_2(\cdot)$ denote the loss function for primal and dual reconstruction tasks, respectively.

\begin{figure}[t]
	\centering
	\includegraphics[width=0.9\columnwidth]{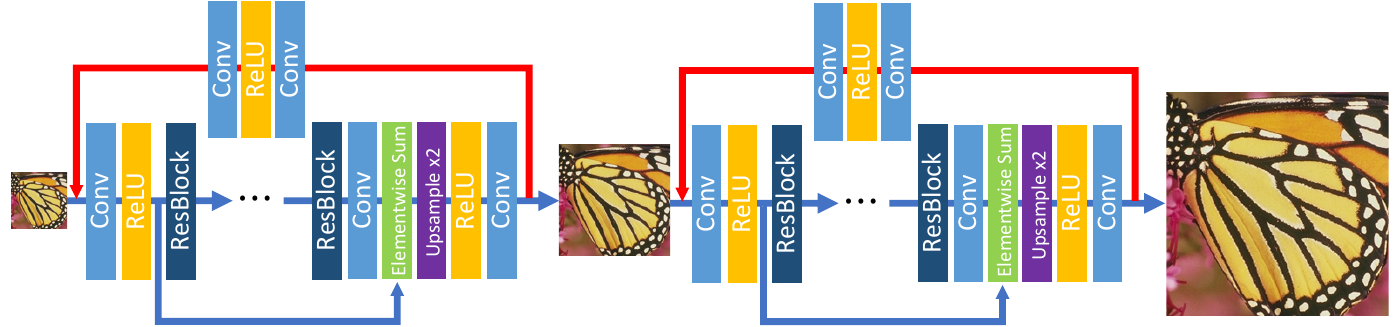}
	\caption{Demonstration of the network architectures for 4$\times$ super-resolution. The blue lines denote the standard reconstruction for the primal model and the red lines denote the backward shortcut for the dual model.}
	\label{fig:architecture}
\end{figure}

\subsection{Progressive dual reconstruction for super-resolution}

\kui{We build our network based on the proposed dual reconstruction scheme to exploit the mutual dependencies between LR and HR images, as shown in Figure~\ref{fig:architecture}.}
Following the design of progressive reconstructions~\cite{wang2015deep, lai2017deep}, the proposed model consists of multiple \emph{dual reconstruction blocks} and progressively predicts the images from low-resolution to high-resolution.
Let $r$ be the upscaling factor, the number of the blocks depends on the upscaling factor: $L = \log_2(r)$. For example, the model contains 2 blocks for 4$\times$ and 3 blocks for 8$\times$ upscaling.

The dual reconstruction scheme can be easily implemented by introducing a backward shortcut connection (see red lines in Figure~\ref{fig:architecture}).
In each block, the primal model $P$ consists of multiple residual modules~\cite{he2016deep} followed by a sub-pixel convolution layer to increase the resolution by $2\times$ upscaling.
Since the dual task aims to learn a much simpler downsampling operation compared to the primal upscaling mapping, the dual model $D$ only contains two convolution layers and a ReLU~\cite{nair2010rectified} activation layer.
During training, we use the bicubic downsampling to resize the ground truth HR image $\by$ to $\by_l$ in $l$-th block. Let $\widehat \by_{l}$ be the predicted image of $l$-th block and $\widehat \by_0 = \by_0$ be the input LR image at the lowest level. For convenience, let $\{\theta_P\}$ and $\{\theta_D\}$ be the parameters for primal and dual models at all levels.
We build a joint loss to receive the supervision at different scales:
\begin{equation}\label{eq:overall_loss}
\mathcal{L}(\widehat \by, \by; \{\theta_P\}, \{\theta_D\}) = \sum_{l=1}^{L} \ell_{\mathrm{DR}}(\widehat \by_{l-1}, \by_{l})
= \sum_{l=1}^{L} \ell_1 \Big( P_l(\widehat \by_{l-1}), \by_{l} \Big) \small{+} \ell_2 \Big( D_l(P_l(\widehat \by_{l-1})), \widehat \by_{l-1} \Big),
\end{equation}
where $P_l$ and $D_l$ denote the primal and dual model in $l$-th block, respectively. We set both $\ell_1$ and $\ell_2$ on the primal and dual reconstruction tasks to the proposed $\ell_{\mathrm{GS}}$ loss function.

\subsection{Theoretical analysis}
We theoretically analyze the generalization bound for the proposed method, where all definitions, proofs and lemmas are put in Appendix A, due to the page limitation.
The generalization error of the dual learning scheme is to measure how accurately the algorithm predicts for the unseen test data in the primal and dual tasks. In particular, we obtain a generalization bound of the proposed model using Rademacher complexity \cite{bartlett2002rademacher}.

\begin{thm} \label{theorem: generalization bound}
	Let $ \ell_1(P(\bx), \by) + \ell_2 (D(P(\bx)), \bx) $ be a mapping from $ \mX \times \mY $ to $ [0, M] $, and the hypothesis set $ \mH_{dual} $ be infinite. Then, for any $ \delta > 0 $, with probability at least $ 1-\delta $, the generalization error $E(P, D)$ (i.e., expected loss)  satisfies for all $ (P, D) \in \mH_{dual} $:
	\begin{align}
	E(P, D) \leq& \widehat{E}(P, D) + 2 R_m^{DL}(\mH_{dual}) + M \sqrt{\frac{1}{2m} \log(\frac{1}{\delta})}, \nonumber \\
	E(P, D) \leq& \widehat{E}(P, D) + 2 \widehat{R}_{\mZ}^{DL}(\mH_{dual}) + 3M \sqrt{\frac{1}{2m} \log(\frac{1}{\delta})}, \nonumber
	\end{align}
	where $m$ is the sample number and $\widehat{E}(P, D)$ is the empirical loss, while $R_m^{DL}$ and $\widehat{R}_{\mZ}^{DL}$ represent the Rademacher complexity and empirical Rademacher complexity of dual learning, respectively.
\end{thm}

This theorem suggests that using the hypothesis set with larger capacity and more samples can guarantee better generalization. We highlight that the derived generalization bound of dual learning, where the loss function is bounded by $ [0, M] $, is more general than \citep{pmlr-v70-xia17a}.
\begin{remark}
	Based on the definition of Rademacher complexity, the capacity of the hypothesis set $ \mH_{dual} \small{\in} \mP \small{\times} \mD $ is smaller than the capacity of hypothesis set $ \mH \small{\in} \mP $ or $ \mH \small{\in} \mD $ in traditional supervised learning, $ i.e. $, $ \widehat{R}_{\mZ}^{DL} \leq \widehat{R}_{\mZ}^{SL} $, where $ \widehat{R}_{\mZ}^{SL} $ is Rademacher complexity defined in supervised learning. In other words, dual reconstruction scheme has a smaller generalization bound than the primal feed-forward scheme and the proposed dual reconstruction model helps the primal model to achieve more accurate SR predictions.\footnote{Experiments on the effectiveness of dual learning scheme can be found in Section~\ref{exp:ablation_dual}.}
\end{remark}


\section{Experiments}

%

\ygyan{
	In the experiments, we perform super-resolution to recover images that are downsampled by factors of 4 and 8, respectively.
	We compare the performance of the proposed method with several state-of-the-art methods on five benchmark datasets,
}
\yong{
	including SET5~\cite{bevilacqua2012low}, SET14~\cite{zeyde2010single}, BSDS100~\cite{arbelaez2011contour}, URBAN100~\cite{huang2015single} and MANGA109~\cite{matsui2017sketch}.
	For quantitative evaluation, we adopt two common image quality metrics, i.e., \emph{PSNR} and \emph{SSIM}~\cite{wang2004image} in the paper.
}

%
%
%

\subsection{Implementation details}\label{Implementation_Details}


\ygyan{
	We train the proposed DRN model using a random subset of 350k images from the ImageNet dataset~\cite{russakovsky2015imagenet}. 
	We randomly crop the input images to $128 \times 128$ RGB images as the HR data, and downsample the HR data using bicubic kernel to obtain the LR data.
	We use ReLU activation in both the primal and dual reconstruction model.
	Each reconstruction block in the primal model has $7$ identical residual modules, i.e., 14 modules for $4\times$ and 21 modules for $8\times$ upscaling.
	We adopt the sub-pixel convolutional layer \cite{shi2016real} to increase the resolution by $2\times$ upscaling.
	The hyperparameter $\lambda$ in Eqn. (\ref{eq:gradient_masked_loss}) is set to $2$.
	During training, we apply the Adam algorithm \cite{kingma2014adam} with $\beta_1 = 0.9$.
	We set minibatch size as 16. 
	The learning rate is initialized to $10^{-5}$ and decreased by a factor of 10 for every $5\times10^{5}$ for total $10^6$ iterations. All experiments were conducted using PyTorch.
}

\begin{figure}[t]
	\centering
	\includegraphics[width = 0.95\columnwidth]{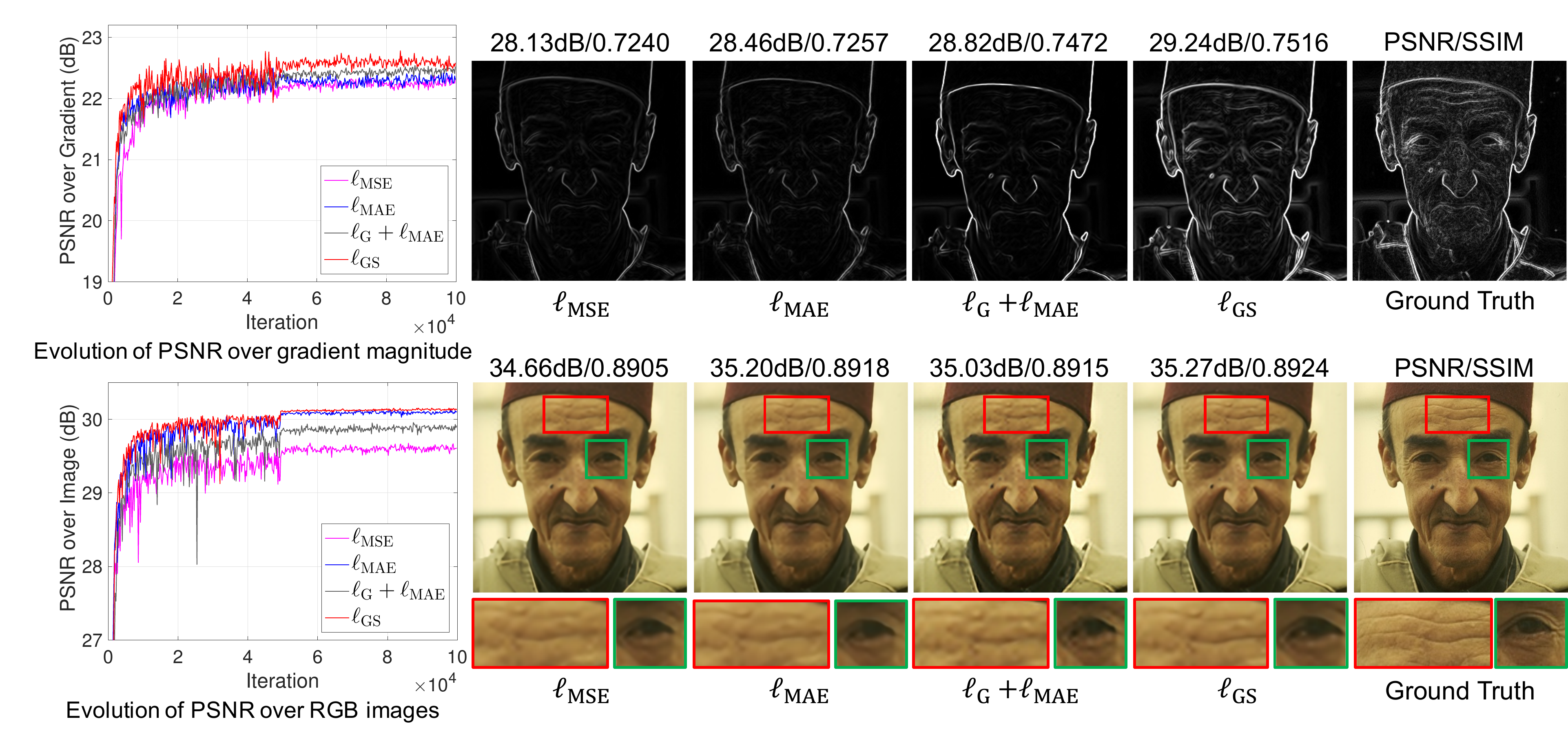}
	\caption{Performance comparison of different loss functions. \ygyan{The PSNR and SSIM values} are shown above the images. \ygyan{The top row denotes recovery results on gradient magnitude.}
	}
	\label{fig:loss_compare}
\end{figure}

\subsection{Demonstration of gradient sensitive loss}\label{Ablation_Study}

\ygyan{
	In this part, we perform super-resolution with $4 \times$ upscaling to study the impacts of different losses, 
	including the Mean Square Error (MSE), Mean Absolute Error (MAE), image gradient loss \cite{mathieu2015deep}, and the proposed gradient-sensitive (GS) loss.
	Figure \ref{fig:loss_compare} presents the results obtained by the different losses.
	The top row denotes the results regarding the image gradient,
	and the bottom row represents the results over the RGB images.
	The proposed gradient-sensitive loss converges to the highest PSNR score among all the compared losses.
	In addition, the gradient magnitude map obtained by $\ell_{\mathrm{GS}}$ is more close to the ground-truth compared with the other losses.
	From the reconstructed RGB images, we observe that $\ell_{\mathrm{GS}}$ is able to capture more details and maintain the perceptual fidelity of the original HR images.
}

%
%

\begin{table*}[tbp]
	\small
	\centering
	\caption{Performance comparison with state-of-the-art algorithms for 4$\times$ upscaling image super-resolution. \textbf{Bold} number indicates the best result and \blue{\textbf{blue}} number indicates the second best result.}
	\resizebox{0.9\textwidth}{!}{
	\begin{tabular}{c|c|c|c|c|c|c|c|c|c|c}
		\toprule
		\multicolumn{1}{c|}{\multirow{2}[0]{*}{Algorithms}} & \multicolumn{2}{c|}{SET5} & \multicolumn{2}{c|}{SET14} & \multicolumn{2}{c|}{BSDS100} & \multicolumn{2}{c|}{URBAN100} & \multicolumn{2}{c}{MANGA109} \\
		& \multicolumn{1}{l}{PSNR} & \multicolumn{1}{c|}{SSIM} & \multicolumn{1}{l}{PSNR} & \multicolumn{1}{c|}{SSIM} & \multicolumn{1}{l}{PSNR} & \multicolumn{1}{c|}{SSIM} & \multicolumn{1}{l}{PSNR} & \multicolumn{1}{c|}{SSIM} & \multicolumn{1}{l}{PSNR} & \multicolumn{1}{c}{SSIM} \\
		\midrule
		Bicubic  &   28.42    &   0.810    &    26.10   &   0.702    &   25.96    &   0.667    &   23.15    &   0.657    &   24.92    & 0.789 \\
		SRCNN~\cite{dong2016image}    &  30.49   &   0.862    &   27.61    &   0.751    &   26.91    &   0.710    &   24.53    &   0.722    &   27.66    &   0.858    \\
		SelfExSR~\cite{huang2015single}   &  30.33    &    0.861   &   27.54    &    0.751   &    26.84   &   0.710    &   24.82    &   0.737    &    27.82   &   0.865    \\
		DRCN~\cite{kim2016deeply}         &    31.53   &   0.885    &   28.04    &   0.767    &   27.24    &   0.723    &   25.14    &   0.751    &   28.97    & 0.886 \\
		ESPCN~\cite{shi2016real}        &    29.21   &   0.851    &   26.40    &   0.744    &   25.50    &   0.696    &   24.02    &   0.726    &   23.55    & 0.795 \\
		SRResNet~\cite{ledig2016photo}        &   32.05    &   0.891    &   28.49    &   0.782    &   27.61    &   0.736    &   26.09    &   0.783    &   30.70    &  0.908\\
		SRGAN~\cite{ledig2016photo}        &   29.46    &   0.838    &   26.60    &   0.718    &   25.74    &  0.666     &   24.50    &   0.736    &   27.79    & 0.856 \\
		FSRCNN~\cite{dong2016accelerating}        &   30.71    &   0.865    &   27.70    &   0.756    &   26.97    &   0.714    &   24.61    &   0.727    &   27.89    & 0.859 \\
		VDSR~\cite{kim2016accurate}    &   31.53    &   0.883    &   28.03    &   0.767    &   27.29    &   0.725    &   25.18    &   0.752    &   28.82    &    0.886  \\
		DRRN~\cite{tai2017image}         &    31.69   &   0.885    &   28.21    &    0.772   &   27.38    &   0.728    &    25.44   &   0.763    &  27.17     & 0.853 \\
		LapSRN~\cite{lai2017deep}        &    31.54   &   0.885    &   28.09    &   0.770    &    27.31   &    0.727   &   25.21    &   0.756    &   29.09    & 0.890 \\
		SRDenseNet~\cite{tong2017image}        &   32.02    &   0.893    &   28.50    &   0.778    &   27.53    &   0.733    &   26.05    &    0.781   &    29.49   & 0.899 \\
		EDSR~\cite{lim2017enhanced}   &    {\textbf{32.46}}   &     \blue{\textbf{0.896}}   &   27.71    &   \blue{\textbf{0.786}}    &   27.72    &   0.742   &   {\textbf{26.64}}    &   {\textbf{0.803}}    &   29.09    &    0.957    \\
		DBPN~\cite{DBPN2018} &    31.76   &     0.887   &   28.39    &   0.778    &   27.48    &   0.733   &   25.71    &  0.772     &   30.22    &   0.902     \\
		\hline
		GS loss (ours)   &   32.17    &   0.895    &   \blue{\textbf{28.51}}    &   0.785    &   \blue{\textbf{27.80}}    &   \blue{\textbf{0.742}}    &   25.95    &   0.789    &   \blue{\textbf{30.91}}    &  \blue{\textbf{0.959}}    \\
		DRN (ours)   &   \blue{\textbf{32.24}}    &   {\textbf{0.897}}    &   {\textbf{28.58}}    &   {\textbf{0.788}}    &   {\textbf{27.86}}    &   {\textbf{0.745}}    &   \blue{\textbf{26.12}}    &   \blue{\textbf{0.792}}    &    {\textbf{30.97}}   &    {\textbf{0.963}}  \\
		\bottomrule
	\end{tabular}
	}
	\label{exp:4xsr}
\end{table*}

\subsection{Comparisons with state-of-the-art \ygyan{methods}}\label{Comparison_with_baselines}

We compare the performance of our proposed DRN approach with several state-of-the-art methods, including
Bicubic, SRCNN~\cite{dong2016image}, SelfExSR~\cite{huang2015single}, DRCN~\cite{kim2016deeply}, ESPCN~\cite{shi2016real}, SRResNet~\cite{ledig2016photo}, SRGAN~\cite{ledig2016photo}, FSRCNN~\cite{dong2016accelerating}, VDSR~\cite{kim2016accurate}, DRRN~\cite{tai2017image}, LapSRN~\cite{lai2017deep}, SRDenseNet~\cite{tong2017image} and EDSR~\cite{lim2017enhanced}.
While preparing this paper, we are aware of a very recent work~\cite{DBPN2018} which shows promising performance. For fair comparison, we train the model using their souce code on our data set with the same setting. However, our reproduced results are worse than the reporting results in~\cite{DBPN2018}. One possible reason is that they use more training data. We study the effect of the number of training data in Figure~\ref{fig:psnr_vs_dataRatio1} of supplementary file.

\ygyan{
	Tables \ref{exp:4xsr} and \ref{exp:8xsr} present the results of $4\times$ and $8\times$ image super-resolution, respectively.
	For the $4 \times$ super-resolution tasks, our proposed GS loss and DRN approach outperform the other conducted methods on most datasets.
	For the $8 \times$ super-resolution tasks, DRN and GS loss achieve the best and the second best performance among all the conducted methods, respectively.
	These observations demonstrate the effectivenes of the proposed methods.
	In addition, DRN with GS loss outperforms GS loss on all the datasets, which validates that the proposed dual reconstruction mechanism is able to further improve the performance.
}
\qi{For further comparison, we provide visual comparisons on some reconstructed images. Figures~\ref{fig:image_compare_4x} shows the $4 \times$ and $8 \times$ SR images obtained by different methods and the corresponding metrics, respectively. We observe that our proposed DRN method consistently achieves the best numerical results and the best visual quality.}

\begin{table*}[tbp]
	\small
	\centering
	\caption{Performance comparison with state-of-the-art algorithms for 8$\times$ upscaling image super-resolution. {\textbf{Bold}} number indicates the best result and \blue{\textbf{blue}} number indicates the second best result.}
	\resizebox{0.9\textwidth}{!}{
	\begin{tabular}{c|c|c|c|c|c|c|c|c|c|c}
		\toprule
		\multicolumn{1}{c|}{\multirow{2}[0]{*}{Algorithms}} & \multicolumn{2}{c|}{SET5} & \multicolumn{2}{c|}{SET14} & \multicolumn{2}{c|}{BSDS100} & \multicolumn{2}{c|}{URBAN100} & \multicolumn{2}{c}{MANGA109} \\
		& \multicolumn{1}{l}{PSNR} & \multicolumn{1}{c|}{SSIM} & \multicolumn{1}{l}{PSNR} & \multicolumn{1}{c|}{SSIM} & \multicolumn{1}{l}{PSNR} & \multicolumn{1}{c|}{SSIM} & \multicolumn{1}{l}{PSNR} & \multicolumn{1}{c|}{SSIM} & \multicolumn{1}{l}{PSNR} & \multicolumn{1}{c}{SSIM} \\
		\midrule
		Bicubic  &   24.39    &   0.657    &    23.19   &   0.568    &   23.67    &   0.547    &   20.74    &   0.515    &   21.47    & 0.649 \\
		SRCNN~\cite{dong2016image}    &  25.33   &   0.689    &   23.85    &   0.593    &   24.13    &   0.565    &   21.29    &   0.543    &   22.37    &   0.682    \\
		SelfExSR~\cite{huang2015single}   &  25.52    &    0.704   &   24.02    &    0.603   &    24.18   &   0.568    &   21.81    &   0.576    &    22.99   &   0.718    \\
		ESPCN~\cite{shi2016real}        &   25.02    &    0.697   &   23.45    &   0.598    &   23.92    &   0.574    &   21.20    &   0.554    &  22.04     &  0.683  \\
		SRResNet~\cite{ledig2016photo}        &   26.62    &   0.756    &   24.55    &   0.624    &   24.65    &   0.587    &   22.05    &   0.589    &   23.88    & 0.748 \\
		SRGAN~\cite{ledig2016photo}        &   23.04    &   0.626    &    21.57   &   0.495    &   21.78    &   0.442    &  19.64     &   0.468    &    20.42   &  0.625 \\
		FSRCNN~\cite{dong2016accelerating}        &   25.41    &   0.682    &   23.93    &   0.592    &   24.21    &   0.567    &   21.32    &   0.537    &   22.39    & 0.672 \\
		VDSR~\cite{kim2016accurate}    &   25.72    &   0.711    &   24.21    &   0.609    &   24.37    &   0.576    &   21.54    &   0.560    &   22.83    &    0.707  \\
		DRRN~\cite{tai2017image}        &    25.76   &   0.721    &   24.21    &   0.583    &   24.47   &   0.533    &   21.02    &  0.530    &   21.88   & 0.663 \\
		LapSRN~\cite{lai2017deep}        &    26.14   &   0.737    &   24.35    &   0.620    &    24.54   &    0.585   &   21.81    &   0.580    &   23.39    & 0.734 \\
		SRDenseNet~\cite{tong2017image}        &   25.99    &   0.704    &   24.23    &   0.581    &   24.45   &    0.530   &    21.67   &   0.562   &   23.09   & 0.712 \\
		EDSR~\cite{lim2017enhanced}        &    26.54   &   0.752    &    24.54   &   0.625    &   24.59   &   0.588    &   22.07    &   0.595   &   23.74   & 0.749 \\
		DBPN~\cite{DBPN2018} &    26.43   &     0.748   &   24.39    &   0.623    &   24.60    &   0.589   &   22.01    &   0.592    &   23.97    &   0.756     \\
		\hline
		GS loss (\ygyan{ours})   &   \blue{\textbf{26.91}}    &   \blue{\textbf{0.772}}    &   \blue{\textbf{24.73}}    &   \blue{\textbf{0.636}}    &   \blue{\textbf{24.70}}    &  \blue{\textbf{0.593}}     &   \blue{\textbf{22.30}}    &   \blue{\textbf{0.609}}    &   \blue{\textbf{24.77}}    &   \blue{\textbf{0.782}}   \\
		DRN (ours)   &   {\textbf{27.03}}    &   {\textbf{0.775}}    &   {\textbf{24.86}}    &   {\textbf{0.641}}    &   {\textbf{24.83}}    &   {\textbf{0.599}}    &   {\textbf{22.46}}    &   {\textbf{0.617}}    &    {\textbf{24.85}}   &    {\textbf{0.790}}  \\
		\bottomrule
	\end{tabular}
	}
	\label{exp:8xsr}
\end{table*}

\begin{figure}[t]
	\centering
	\includegraphics[width = 0.75\columnwidth]{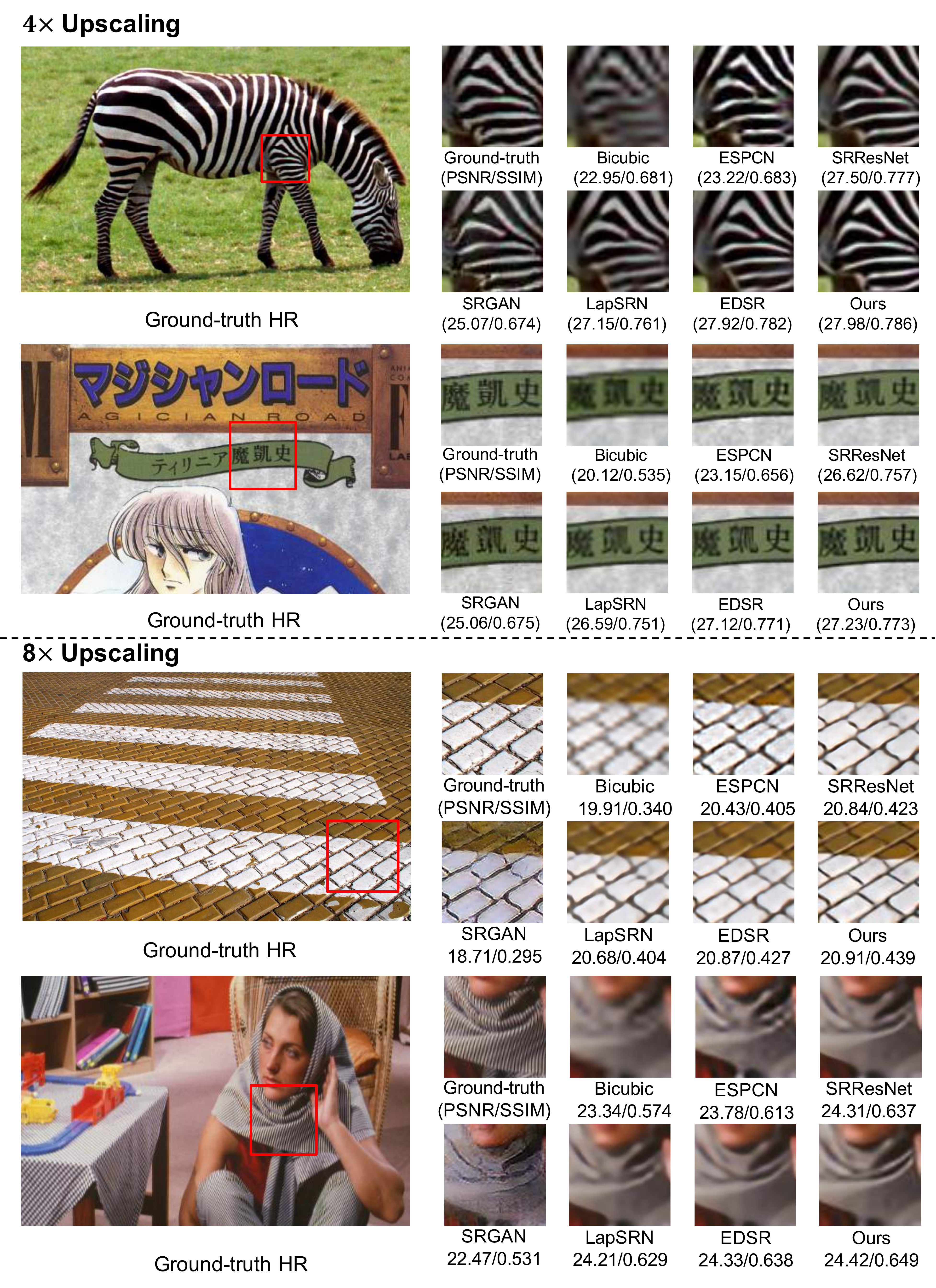}
	\caption{Visual comparison for $4\times$ and $8\times$ image super-resolution on benchmark datasets.}
	\label{fig:image_compare_4x}
\end{figure}


\begin{table}[hbp]
	\small
	\centering
	\caption{Performance comparison over image gradient in terms of PSNR.  [8$\times$ upscaling]}
	\begin{tabular}{c|c|c|c|c|c}
		\toprule
		\multicolumn{1}{c|}{Algorithms}   & \multicolumn{1}{c|}{Set5} & \multicolumn{1}{c|}{Set14} & \multicolumn{1}{c|}{BSDS100} & \multicolumn{1}{c|}{Urban100} & \multicolumn{1}{c}{Manga109} \\
		\midrule
		SRGAN~\cite{ledig2016photo} & 20.01    & 19.48 & 19.59 & 18.58 & 19.54 \\
		SRResNet~\cite{ledig2016photo} & 20.82 & 19.65 & 20.50  & 18.87 & 20.43 \\
		LapSRN~\cite{lai2017deep} & 20.14 & 19.29 & 20.36 & 18.48 & 18.96 \\
		EDSR~\cite{lim2017enhanced}  & 19.84 & 19.31 & 19.98 & 18.39 & 20.16 \\
		DBPN~\cite{DBPN2018}  & 20.93 & 19.76 & 20.47 & 18.87 & 20.50 \\
		DRN (Ours)  & \textbf{21.29} & \textbf{20.04} & \textbf{20.62} & \textbf{19.11} & \textbf{20.68} \\
		\bottomrule
	\end{tabular}%
	\label{tab:psnr_gradient}
\end{table}

\begin{table}[hbp]
	\small
	\centering
	\caption{Abalation study of dual reconstruction scheme and progressive structure. We report \ygyan{the} PSNR \ygyan{scores} on \ygyan{the} SET5 and SET14 datasets.}
	\begin{tabular}{c|c|c|c|c}
		\toprule
		Method & Plain & Dual & Progressive & Dual + Progressive \\
		\midrule
		SET5 & 31.96 & 32.04 & 32.17 & \textbf{32.24} \\
		SET14 & 28.37 & 28.47 & 28.52 & \textbf{28.58} \\
		\bottomrule
	\end{tabular}%
	\label{exp:dual}%
\end{table}%

\section{More results and discussions}

\subsection{Comparisons of PSNR over image gradient}
\ygyan{
	Table~\ref{tab:psnr_gradient} lists the PSNR scores over the image gradient on several benchmark datasets. Our proposed DRN achieves the best performance, which demonstrates that the DRN network has a better ability to capture the structural information compared with the other methods.
}

\subsection{Effects of dual reconstruction scheme and progressive structure.}\label{exp:ablation_dual}
\ygyan{
	In this experiment, we evaluate the effects of the dual reconstruction scheme and conduct analysis on the progressive structure.
	The ``non-progressive'' methods directly predict the final HR images without the supervision from the prediction of intermediate images.
	The ``non-dual'' learning methods remove the dual learning part and fall back to plain feed-forward methods.
	Table \ref{exp:dual} shows the PSNR scores of the $4 \times$ super-resolution tasks on the SET5 and SET14 datasets.
	We observe that both the progressive and dual methods outperform the plain methods (``non-progressive'' and ``non-dual'')
	The combination of dual reconstruction scheme and progressive structure method achieves the best performance. 
	These results demonstrate the efficacy of the proposed progressive reconstruction and dual learning approaches. 
}

\section{Conclusion}
In this work, we propose a novel gradient sensitive loss (GS) to capture both low-frequency content and high-frequency structure for image super-resolution. Moreover, to exploit the mutual dependencies between LR and HR images, we propose a dual reconstruction to further improve the performance. Our model is trained with the proposed GS loss in a progressive coarse-to-fine manner.
More critically, we conduct theoritical analysis on the generalization bound of the proposed method. Extensive experiments demonstrate that the proposed method produces perceptually sharper images and significantly outperforms the state-of-the-art SR methods with a large upscaling factor of $4\times$ and $8\times$.

\clearpage
{
	\bibliographystyle{abbrv}


\begin{thebibliography}{26}
	
	\bibitem{allebach1996edge}
	Jan Allebach and Ping~Wah Wong.
	\newblock Edge-directed interpolation.
	\newblock In {\em Image Processing, 1996. Proceedings., International
		Conference on}, volume~3, pages 707--710. IEEE, 1996.
	
	\bibitem{arbelaez2011contour}
	Pablo Arbelaez, Michael Maire, Charless Fowlkes, and Jitendra Malik.
	\newblock Contour detection and hierarchical image segmentation.
	\newblock {\em IEEE transactions on pattern analysis and machine intelligence},
	33(5):898--916, 2011.
	
	\bibitem{baker2002limits}
	Simon Baker and Takeo Kanade.
	\newblock Limits on super-resolution and how to break them.
	\newblock {\em IEEE Transactions on Pattern Analysis and Machine Intelligence},
	24(9):1167--1183, 2002.
	
	\bibitem{bartlett2002rademacher}
	Peter~L Bartlett and Shahar Mendelson.
	\newblock Rademacher and gaussian complexities: Risk bounds and structural
	results.
	\newblock {\em Journal of Machine Learning Research}, 3(Nov):463--482, 2002.
	
	\bibitem{ben2007penrose}
	Moshe Ben-Ezra, Zhouchen Lin, and Bennett Wilburn.
	\newblock Penrose pixels super-resolution in the detector layout domain.
	\newblock In {\em Computer Vision, 2007. ICCV 2007. IEEE 11th International
		Conference on}, pages 1--8. IEEE, 2007.
	
	\bibitem{bevilacqua2012low}
	Marco Bevilacqua, Aline Roumy, Christine Guillemot, and Marie~Line
	Alberi-Morel.
	\newblock Low-complexity single-image super-resolution based on nonnegative
	neighbor embedding.
	\newblock 2012.
	
	\bibitem{chowdhuri2012very}
	Debabrata Chowdhuri, KS~Sendhil~Kumar, M~Rajasekhara Babu, and Ch~Pradeep
	Reddy.
	\newblock Very low resolution face recognition in parallel environment.
	\newblock {\em IJCSIT) International Journal of Computer Science and
		Information Technologies}, 3(3):4408--4410, 2012.
	
	\bibitem{dong2014learning}
	Chao Dong, Chen~Change Loy, Kaiming He, and Xiaoou Tang.
	\newblock Learning a deep convolutional network for image super-resolution.
	\newblock In {\em European Conference on Computer Vision}, pages 184--199.
	Springer, 2014.
	
	\bibitem{dong2016image}
	Chao Dong, Chen~Change Loy, Kaiming He, and Xiaoou Tang.
	\newblock {I}mage {S}uper-resolution using {D}eep {C}onvolutional {N}etworks.
	\newblock {\em IEEE transactions on pattern analysis and machine intelligence},
	38(2):295--307, 2016.
	
	\bibitem{dong2016accelerating}
	Chao Dong, Chen~Change Loy, and Xiaoou Tang.
	\newblock Accelerating the super-resolution convolutional neural network.
	\newblock In {\em European Conference on Computer Vision}, pages 391--407.
	Springer, 2016.
	
	\bibitem{gao2012image}
	Xinbo Gao, Kaibing Zhang, Dacheng Tao, and Xuelong Li.
	\newblock Image super-resolution with sparse neighbor embedding.
	\newblock {\em IEEE Transactions on Image Processing}, 21(7):3194--3205, 2012.
	
	\bibitem{gu2015convolutional}
	Shuhang Gu, Wangmeng Zuo, Qi~Xie, Deyu Meng, Xiangchu Feng, and Lei Zhang.
	\newblock Convolutional sparse coding for image super-resolution.
	\newblock In {\em Proceedings of the IEEE International Conference on Computer
		Vision}, pages 1823--1831, 2015.
	
	\bibitem{DBPN2018}
	Muhammad Haris, Greg Shakhnarovich, and Norimichi Ukita.
	\newblock Deep back-projection networks for super-resolution.
	\newblock 2018.
	
	\bibitem{he2016dual}
	Di~He, Yingce Xia, Tao Qin, Liwei Wang, Nenghai Yu, Tieyan Liu, and Wei-Ying
	Ma.
	\newblock Dual learning for machine translation.
	\newblock In {\em Advances in Neural Information Processing Systems}, pages
	820--828, 2016.
	
	\bibitem{he2016deep}
	Kaiming He, Xiangyu Zhang, Shaoqing Ren, and Jian Sun.
	\newblock {D}eep {R}esidual {L}earning for {I}mage {R}ecognition.
	\newblock In {\em Proceedings of the IEEE Conference on Computer Vision and
		Pattern Recognition}, pages 770--778, 2016.
	
	\bibitem{hou1978cubic}
	Hsieh Hou and H~Andrews.
	\newblock Cubic splines for image interpolation and digital filtering.
	\newblock {\em IEEE Transactions on acoustics, speech, and signal processing},
	26(6):508--517, 1978.
	
	\bibitem{huang2015single}
	Jia-Bin Huang, Abhishek Singh, and Narendra Ahuja.
	\newblock Single image super-resolution from transformed self-exemplars.
	\newblock In {\em Proceedings of the IEEE Conference on Computer Vision and
		Pattern Recognition}, pages 5197--5206, 2015.
	
	\bibitem{hui2018fast}
	Zheng Hui, Xiumei Wang, and Xinbo Gao.
	\newblock Fast and accurate single image super-resolution via information
	distillation network.
	\newblock {\em IEEE Conference on Computer Vision and Pattern Recognition
		(CVPR)}, 2018.
	
	\bibitem{jianchao2008image}
	Yang Jianchao, John Wright, Thomas Huang, and Yi~Ma.
	\newblock Image super-resolution as sparse representation of raw image patches.
	\newblock In {\em Proc. IEEE Conf. on Computer Vision and Pattern Recognition},
	pages 1--8, 2008.
	
	\bibitem{johnson2016perceptual}
	Justin Johnson, Alexandre Alahi, and Li~Fei-Fei.
	\newblock {P}erceptual {L}osses for {R}eal-time {S}tyle {T}ransfer and
	{S}uper-resolution.
	\newblock In {\em European Conference on Computer Vision}, pages 694--711.
	Springer, 2016.
	
	\bibitem{kim2016accurate}
	Jiwon Kim, Jung Kwon~Lee, and Kyoung Mu~Lee.
	\newblock Accurate image super-resolution using very deep convolutional
	networks.
	\newblock In {\em Proceedings of the IEEE Conference on Computer Vision and
		Pattern Recognition}, pages 1646--1654, 2016.
	
	\bibitem{kim2016deeply}
	Jiwon Kim, Jung Kwon~Lee, and Kyoung Mu~Lee.
	\newblock Deeply-recursive convolutional network for image super-resolution.
	\newblock In {\em Proceedings of the IEEE conference on computer vision and
		pattern recognition}, pages 1637--1645, 2016.
	
	\bibitem{kingma2014adam}
	Diederik Kingma and Jimmy Ba.
	\newblock {A}dam: {A} method for stochastic optimization.
	\newblock {\em International Conference on Learning Representations}, 2015.
	
	\bibitem{lai2017deep}
	Wei-Sheng Lai, Jia-Bin Huang, Narendra Ahuja, and Ming-Hsuan Yang.
	\newblock Deep laplacian pyramid networks for fast and accurate
	super-resolution.
	\newblock In {\em Proc. IEEE Conf. Comput. Vis. Pattern Recognit.}, pages
	624--632, 2017.
	
	\bibitem{ledig2016photo}
	Christian Ledig, Lucas Theis, Ferenc Husz{\'a}r, Jose Caballero, Andrew
	Cunningham, Alejandro Acosta, Andrew Aitken, Alykhan Tejani, Johannes Totz,
	Zehan Wang, et~al.
	\newblock Photo-realistic {S}ingle {I}mage {S}uper-resolution using a
	{G}enerative {A}dversarial {N}etwork.
	\newblock {\em IEEE Conference on Computer Vision and Pattern Recognition
		(CVPR)}, 2017.
	
	\bibitem{li2001new}
	Xin Li and Michael~T Orchard.
	\newblock New edge-directed interpolation.
	\newblock {\em IEEE transactions on image processing}, 10(10):1521--1527, 2001.
	
	\bibitem{liao2015video}
	Renjie Liao, Xin Tao, Ruiyu Li, Ziyang Ma, and Jiaya Jia.
	\newblock Video super-resolution via deep draft-ensemble learning.
	\newblock In {\em Proceedings of the IEEE International Conference on Computer
		Vision}, pages 531--539, 2015.
	
	\bibitem{lim2017enhanced}
	Bee Lim, Sanghyun Son, Heewon Kim, Seungjun Nah, and Kyoung~Mu Lee.
	\newblock Enhanced deep residual networks for single image super-resolution.
	\newblock In {\em The IEEE Conference on Computer Vision and Pattern
		Recognition (CVPR) Workshops}, volume~1, page~3, 2017.
	
	\bibitem{lin2004fundamental}
	Zhouchen Lin and Heung-Yeung Shum.
	\newblock Fundamental limits of reconstruction-based superresolution algorithms
	under local translation.
	\newblock {\em IEEE Transactions on Pattern Analysis and Machine Intelligence},
	26(1):83--97, 2004.
	
	\bibitem{NIPS2016_6172}
	Xiaojiao Mao, Chunhua Shen, and Yu-Bin Yang.
	\newblock Image restoration using very deep convolutional encoder-decoder
	networks with symmetric skip connections.
	\newblock In D.~D. Lee, M.~Sugiyama, U.~V. Luxburg, I.~Guyon, and R.~Garnett,
	editors, {\em Advances in Neural Information Processing Systems 29}, pages
	2802--2810. Curran Associates, Inc., 2016.
	
	\bibitem{mathieu2015deep}
	Michael Mathieu, Camille Couprie, and Yann LeCun.
	\newblock {D}eep {M}ulti-scale {V}ideo {P}rediction beyond {M}ean {S}quare
	{E}rror.
	\newblock {\em International Conference on Learning Representations}, 2016.
	
	\bibitem{matsui2017sketch}
	Yusuke Matsui, Kota Ito, Yuji Aramaki, Azuma Fujimoto, Toru Ogawa, Toshihiko
	Yamasaki, and Kiyoharu Aizawa.
	\newblock Sketch-based manga retrieval using manga109 dataset.
	\newblock {\em Multimedia Tools and Applications}, 76(20):21811--21838, 2017.
	
	\bibitem{mohri2012foundations}
	Mehryar Mohri, Afshin Rostamizadeh, and Ameet Talwalkar.
	\newblock {\em Foundations of machine learning}.
	\newblock MIT press, 2012.
	
	\bibitem{nair2010rectified}
	Vinod Nair and Geoffrey~E Hinton.
	\newblock Rectified linear units improve restricted boltzmann machines.
	\newblock In {\em Proceedings of the 27th international conference on machine
		learning (ICML-10)}, pages 807--814, 2010.
	
	\bibitem{nehme2018deep}
	Elias Nehme, Lucien~E Weiss, Tomer Michaeli, and Yoav Shechtman.
	\newblock Deep-storm: super-resolution single-molecule microscopy by deep
	learning.
	\newblock {\em Optica}, 5(4):458--464, 2018.
	
	\bibitem{nguyen2000efficient}
	Nhat Nguyen and Peyman Milanfar.
	\newblock An efficient wavelet-based algorithm for image superresolution.
	\newblock In {\em Image Processing, 2000. Proceedings. 2000 International
		Conference on}, volume~2, pages 351--354. IEEE, 2000.
	
	\bibitem{russakovsky2015imagenet}
	Olga Russakovsky, Jia Deng, Hao Su, Jonathan Krause, Sanjeev Satheesh, Sean Ma,
	Zhiheng Huang, Andrej Karpathy, Aditya Khosla, Michael Bernstein, et~al.
	\newblock {I}magenet {L}arge {S}cale {V}isual {R}ecognition {C}hallenge.
	\newblock {\em International Journal of Computer Vision}, 115(3):211--252,
	2015.
	
	\bibitem{shi2016real}
	Wenzhe Shi, Jose Caballero, Ferenc Husz{\'a}r, Johannes Totz, Andrew~P Aitken,
	Rob Bishop, Daniel Rueckert, and Zehan Wang.
	\newblock Real-time {S}ingle {I}mage and {V}ideo {S}uper-resolution using an
	{E}fficient {S}ub-pixel {C}onvolutional {N}eural {N}etwork.
	\newblock In {\em Proceedings of the IEEE Conference on Computer Vision and
		Pattern Recognition}, pages 1874--1883, 2016.
	
	\bibitem{simonyan2014very}
	Karen Simonyan and Andrew Zisserman.
	\newblock Very deep convolutional networks for large-scale image recognition.
	\newblock {\em arXiv preprint arXiv:1409.1556}, 2014.
	
	\bibitem{tai2017image}
	Ying Tai, Jian Yang, and Xiaoming Liu.
	\newblock Image super-resolution via deep recursive residual network.
	\newblock In {\em The IEEE Conference on Computer Vision and Pattern
		Recognition (CVPR)}, volume~1, 2017.
	
	\bibitem{timofte2013anchored}
	Radu Timofte, Vincent De, and Luc Van~Gool.
	\newblock Anchored neighborhood regression for fast example-based
	super-resolution.
	\newblock In {\em Computer Vision (ICCV), 2013 IEEE International Conference
		on}, pages 1920--1927. IEEE, 2013.
	
	\bibitem{tong2017image}
	Tong Tong, Gen Li, Xiejie Liu, and Qinquan Gao.
	\newblock Image super-resolution using dense skip connections.
	\newblock In {\em 2017 IEEE International Conference on Computer Vision
		(ICCV)}, pages 4809--4817. IEEE, 2017.
	
	\bibitem{wang2018recovering}
	Xintao Wang, Ke~Yu, Chao Dong, and Chen~Change Loy.
	\newblock Recovering realistic texture in image super-resolution by deep
	spatial feature transform.
	\newblock {\em IEEE Conference on Computer Vision and Pattern Recognition
		(CVPR)}, 2018.
	
	\bibitem{wang2015deep}
	Zhaowen Wang, Ding Liu, Jianchao Yang, Wei Han, and Thomas Huang.
	\newblock Deep networks for image super-resolution with sparse prior.
	\newblock In {\em Proceedings of the IEEE International Conference on Computer
		Vision}, pages 370--378, 2015.
	
	\bibitem{wang2004image}
	Zhou Wang, Alan~C Bovik, Hamid~R Sheikh, and Eero~P Simoncelli.
	\newblock Image quality assessment: from error visibility to structural
	similarity.
	\newblock {\em IEEE transactions on image processing}, 13(4):600--612, 2004.
	
	\bibitem{pmlr-v70-xia17a}
	Yingce Xia, Tao Qin, Wei Chen, Jiang Bian, Nenghai Yu, and Tie-Yan Liu.
	\newblock Dual supervised learning.
	\newblock In Doina Precup and Yee~Whye Teh, editors, {\em Proceedings of the
		34th International Conference on Machine Learning}, volume~70 of {\em
		Proceedings of Machine Learning Research}, pages 3789--3798, International
	Convention Centre, Sydney, Australia, 06--11 Aug 2017. PMLR.
	
	\bibitem{xia2017dual}
	Yingce Xia, Tao Qin, Wei Chen, Jiang Bian, Nenghai Yu, and Tie-Yan Liu.
	\newblock Dual supervised learning.
	\newblock {\em arXiv preprint arXiv:1707.00415}, 2017.
	
	\bibitem{yang2014single}
	Chih-Yuan Yang, Chao Ma, and Ming-Hsuan Yang.
	\newblock Single-image super-resolution: A benchmark.
	\newblock In {\em European Conference on Computer Vision}, pages 372--386.
	Springer, 2014.
	
	\bibitem{zeyde2010single}
	Roman Zeyde, Michael Elad, and Matan Protter.
	\newblock On single image scale-up using sparse-representations.
	\newblock In {\em International conference on curves and surfaces}, pages
	711--730. Springer, 2010.
	
	\bibitem{zhang2017learning}
	Kai Zhang, Wangmeng Zuo, and Lei Zhang.
	\newblock Learning a single convolutional super-resolution network for multiple
	degradations.
	\newblock {\em IEEE Conference on Computer Vision and Pattern Recognition
		(CVPR)}, 2018.
	
	\bibitem{zhang2018residual}
	Yulun Zhang, Yapeng Tian, Yu~Kong, Bineng Zhong, and Yun Fu.
	\newblock Residual dense network for image super-resolution.
	\newblock {\em IEEE Conference on Computer Vision and Pattern Recognition
		(CVPR)}, 2018.
	
\end{thebibliography}
}

\newpage
\appendix

\begin{center}
{
	\Large{\textbf{Supplementary Materials for ``Dual Reconstruction Nets for Image Super-Resolution with Gradient Sensitive Loss''}}
}
\end{center}

\section{Theoretical analysis}\label{sec:proof}
In this section, we will analyze the generalization bound for the proposed method.
The generalization error of the dual learning scheme is to measure how accurately the algorithm predicts for the unseen test data in the primal and dual tasks. Firstly, we will introduce the definition of the generalization error as follows:
\begin{deftn}
	Given an underlying distribution $ \mS $ and hypotheses $ P \in \mP $ and $ D \in \mD $ for the primal and dual tasks, where $ \mP = \{ P_{\theta_{\bx\by}}(\bx); \theta_{\bx\by} \in \Theta_{\bx\by} \} $ and $ \mD = \{ D_{\theta_{\by\bx}}(\by); \theta_{\by\bx} \in \Theta_{\by\bx} \}$, and $ \Theta_{\bx\by} $ and $ \Theta_{\by\bx} $ are parameter spaces, respectively, the generalization error (expected loss) of h is defined by:
	\begin{align*}
		E(P, D) = \mmE_{(\bx, \by) \sim \mP} \left[ \ell_1(P(\bx), \by) + \ell_2 (D(P(\bx)), \bx) \right], \; \forall P \in \mP, D \in \mD.
	\end{align*}
\end{deftn} 

In practice, the goal of the dual learning is to optimize the bi-directional tasks. For any $ P \in \mP $ and $ D \in \mD $, we define the empirical loss on the $ m $ samples as follows:
\begin{align}
	\widehat{E}(P, D) = \frac{1}{m} \sum_{i=1}^{m} \ell_1(P(\bx_i), \by_i) + \ell_2 (D(P(\bx_i)), \bx_i)
\end{align}
Following \cite{bartlett2002rademacher}, we define Rademacher complexity for dual learning in this paper.
We define the hypothesis set as $ \mH_{dual} \in \mP \times \mD $, this Rademacher complexity can measure the complexity of the hypothesis set, that is it can capture the richness of a family of the primal and the dual models.
For our application, we mildly rewrite the definition of Rademacher complexity in  \cite{mohri2012foundations} as follows:
\begin{deftn} \textbf{\emph{(Rademacher complexity of dual learning) }} \label{Def: Rademacher Complexity}
	Given an underlying distribution $ \mS $, and its empirical distribution $ \mZ = \{\bz_1, \bz_2, \cdots,\bz_m\} $, where $ \bz_i = (\bx_i, \by_i) $, then the Rademacher complexity of dual learning is defined as:
	\begin{align*}
		R_m^{DL} (\mH_{dual}) = \mmE_{\mZ} \left[ \widehat{R}_{\mZ} (P, D) \right], \; \forall P \in \mP, D \in \mD,
	\end{align*}
	where $ \widehat{R}_{\mZ} (P, D) $ is its empirical Rademacher complexity defined as:
	\begin{align*}
		\widehat{R}_{\mZ} (P, D) = \mmE_{\sigma} \left[ \sup_{(P, D) \in \mH_{dual}}  \frac{1}{m} \sum_{i=1}^{m} \sigma_i (\ell_1(P(\bx_i), \by_i) + \ell_2 (D(P(\bx_i)), \bx_i))  \right].
	\end{align*}
	where $ \sigma = \{ \sigma_1, \sigma_2, \cdots, \sigma_m \} $ are independent uniform $ \{\pm 1\} $-valued random variables with $ p(\sigma_i = 1) = p(\sigma_i = -1) = \frac{1}{2} $.
\end{deftn}


\subsection{Generalization bound}
This subsection give a generalization guarantees for the dual learning problem. We start with a simple case of a finite hypothesis set.
\begin{thm}
	Let $ [\ell_1(P(\bx), \by) + \ell_2 (D(P(\bx)), \bx)] $ be a mapping from $ \mX \times \mY $ to $ [0, M] $, and suppose the hypothesis set $ \mH_{dual} $ is finite, then for any $ \delta > 0 $, with probability at least $ 1-\delta $, the following inequality holds for all $ (P, D) \in \mH_{dual} $:
	\begin{align*}
		E(P, D) \leq \widehat{E}(P, D) + M \sqrt{\frac{\log|\mH_{dual}| + \log \frac{1}{\delta}}{2m}}.
	\end{align*}
\end{thm}
\begin{proof}
	Based on Hoeffding's inequality, since $ [\ell_1(P(\bx), \by) + \ell_2 (D(P(\bx)), \bx)] $ is bounded in $ [0, M] $, for any $ (P, D) \in \mH_{dual} $, then
	\begin{align*}
		P\left[ E(P, D) - \widehat{E}(P, D) > \epsilon \right] \leq e^{-\frac{2m\epsilon^2}{M^2}}
	\end{align*}
	Based on the union bound, we have
	\begin{align*}
		&P\left[ \exists (P, D) \in \mH_{dual}: E(P, D) - \widehat{E}(P, D) > \epsilon \right] \\
		\leq& \sum_{(P, D) \in \mH_{dual}} P \left[ E(P, D) - \widehat{E}(P, D) > \epsilon \right] \\
		\leq& |\mH_{dual}| e^{-\frac{2m\epsilon^2}{M^2}}.
	\end{align*}
	Let $ |\mH_{dual}| e^{-\frac{2m\epsilon^2}{M^2}} = \delta $, we have $ \epsilon = M \sqrt{\frac{\log|\mH_{dual}| + \log \frac{1}{\delta}}{2m}} $ and conclude the theorem.
\end{proof}
This theorem shows that a larger sample size $ m $ and smaller hypothesis set can guarantee the generalization.
Next we will give a generalization bound of a general case of infinite hypothesis sets using Rademacher complexity.	

\begin{thm} \label{theorem: generalization bound}
	Let $ \ell_1(P(\bx), \by) + \ell_2 (D(P(\bx)), \bx) $ be a mapping from $ \mX \times \mY $ to $ [0, M] $, then for any $ \delta > 0 $, with probability at least $ 1-\delta $, the following inequality holds for all $ (P, D) \in \mH_{dual} $:
	\begin{align}
		E(P, D) \leq& \widehat{E}(P, D) + 2 R_m^{DL} + M \sqrt{\frac{1}{2m} \log(\frac{1}{\delta})}\\
		E(P, D) \leq& \widehat{E}(P, D) + 2 \widehat{R}_{\mZ}^{DL} + 3M \sqrt{\frac{1}{2m} \log(\frac{1}{\delta})}.
	\end{align}
\end{thm}
\begin{proof}
	Based on Theorem 3.1 in \citep{mohri2012foundations}, we extend a case for $ \ell_1(P(\bx), \by) + \ell_2 (D(P(\bx)), \bx) $ bounded in $ [0, M] $.
\end{proof}
Theorem \ref{theorem: generalization bound} shows that with probability at least $ 1-\delta $, the generalization error is smaller than $ 2 R_m^{DL} + M \sqrt{\frac{1}{2m} \log(\frac{1}{\delta})} $ or $ 2 \widehat{R}_{\mZ}^{DL} + 3M \sqrt{\frac{1}{2m} \log(\frac{1}{\delta})} $.
It suggests that using the hypothesis set with larger capacity and more samples can guarantee better generalization.
Moreover, the generalization bound of dual learning is more general for the case that the loss function $ \ell_1(P(\bx), \by) + \ell_2 (D(P(\bx)), \bx) $ is bounded by $ [0, M] $, which is different from \citep{pmlr-v70-xia17a}.

\begin{remark}
	Based on the definition of Rademacher complexity, the capacity of the hypothesis set $ \mH_{dual} \small{\in} \mP \small{\times} \mD $ is smaller than the capacity of hypothesis set $ \mH \small{\in} \mP $ or $ \mH \small{\in} \mD $ in traditional supervised learning, $ i.e. $, $ \widehat{R}_{\mZ}^{DL} \leq \widehat{R}_{\mZ}^{SL} $, where $ \widehat{R}_{\mZ}^{SL} $ is Rademacher complexity defined in supervised learning. In other words, dual learning has a smaller generalization bound than supervised learning and the proposed dual reconstruction model helps the primal model to achieve more accurate SR predictions.
\end{remark}

\section{Discussions}

\subsection{Demonstration of gradient sensitive loss}

\qi{For better understanding, we plot some results about our manipulation of image gradient in Figure~\ref{fig:gradient_sensitive}. By visualizing the gradient $\nabla \mathbf{I}$ of the image $\mathbf{I}$, we can observe the structure information directly. Meanwhile, $\mathbf{M}\odot\mathbf{I}$ means the mask $\mathbf{M}$ has the pixel-wise multiplication with image $\mathbf{I}$ and $(\mathbf{1}-\mathbf{M})\odot\mathbf{I}$ represent the rest part. In our proposed method, we use the gradient of $\mathbf{M}\odot\mathbf{I}$, \textit{i.e.}  $\nabla(\mathbf{M}\odot\mathbf{I})$, to compute the loss function.}

\begin{figure}[htbp]
	\centering
	\includegraphics[trim = 0mm 15mm 0mm 0mm,
			clip,width = 1\columnwidth]{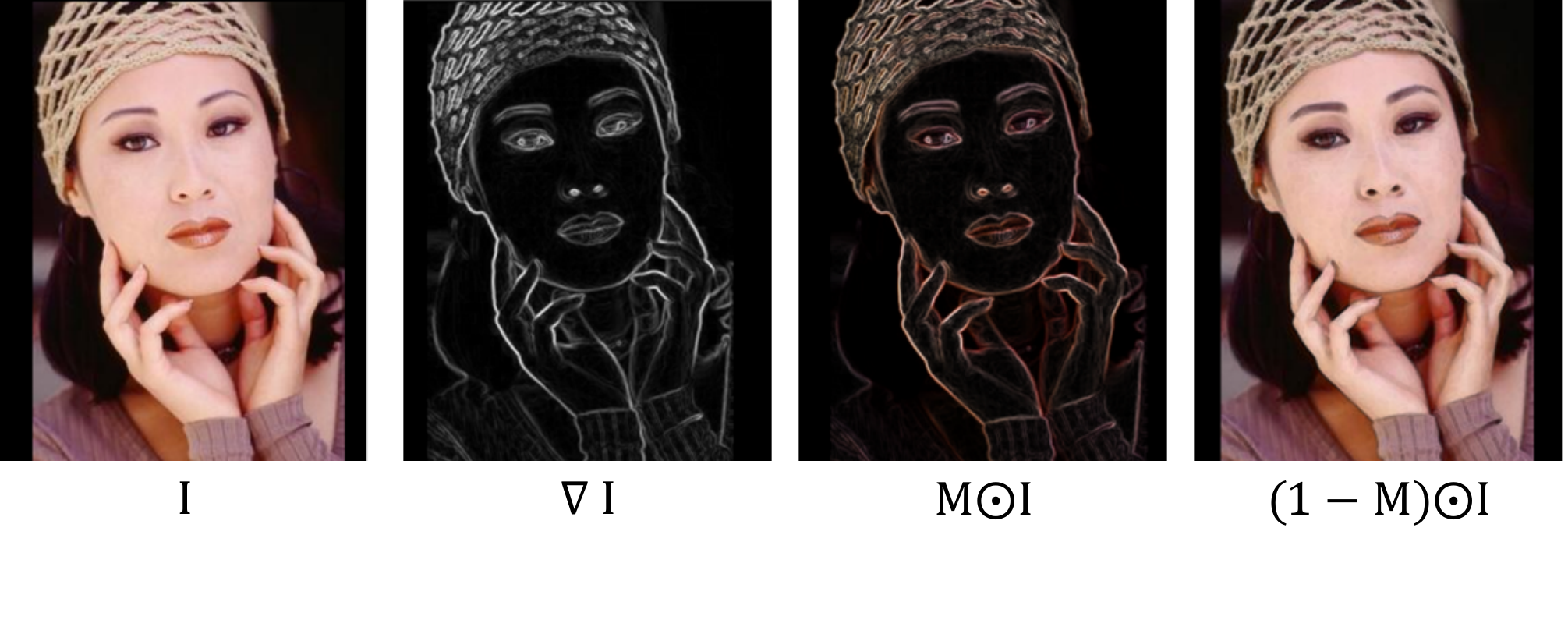}
	\caption{Demonstration of the mask $\bM$.}
	\label{fig:gradient_sensitive}
\end{figure}

\subsection{Comparisons of different loss functions and training schemes}


Table \ref{tab:loss_compare} summarizes the characteristics of different loss functions and training schemes.
In particular, the standard training scheme only forces the model to match the HR images, while the dual scheme receives the supervised information from both LR and HR images.

For the different objective functions, the perceptual loss does not obtain frequency information from images, the gradient loss only captures the high-frequency information, and the MAE, MSE and adversarial loss only btains the low-frequency information. In comparison, our proposed gradient-sensitive loss ($\ell_{\mathrm{GS}}$) is able to capture both the low- and high-frequency information from images.

Overall, the proposed DRN method with the dual scheme is able to exploit both the low- and high-frequency information, and receive supervision form both LR and HR images.

\begin{table}[htbp]
	\centering
	\caption{Comparisons of different objectives and training schemes in Super-Resolution. $\checkmark$ denotes YES, while blank denotes NO.}
	\begin{tabular}{c|c|c|c|c|c}
		\hline
		\multirow{2}[0]{*}{Schemes} & \multicolumn{1}{c|}{\multirow{2}[0]{*}{Methods}} & \multicolumn{2}{c|}{Supervision} &
		\multicolumn{2}{c}{Information} \\
		\cline{3-6}
		&       & \multicolumn{1}{c|}{LR images} & \multicolumn{1}{c|}{HR images} & \multicolumn{1}{c|}{low-frequency} & \multicolumn{1}{l}{high-frequency} \\
		\hline
		\multirow{5}[5]{*}{Standard} & MAE  &       &   \checkmark    &   \checkmark    &  \\
		& MSE &       &   \checkmark    &   \checkmark    &  \\
		& Gradient loss &       &   \checkmark    &       &  \checkmark \\
		& Adversarial loss  &       &    \checkmark   &   \checkmark    &  \\
		& Percputal loss &       &   \checkmark    &       &  \\
		& $\ell_{\mathrm{GS}}$ loss (ours)  &       &   \checkmark    &    \checkmark   & \checkmark \\
		\hline
		\multirow{1}[0]{*}{Dual}
		& $\ell_{\mathrm{GS}}$ loss (ours)  &   \checkmark    &   \checkmark    &   \checkmark    & \checkmark \\
		\hline
	\end{tabular}%
	\label{tab:loss_compare}%
\end{table}%

\section{Experimental results}


\subsection{Effect of $\lambda$ in Eqn.~(\ref{eq:gradient_masked_loss})}

In this experiment, we study the performance of our proposed DRN method under different values of the parameter $\lambda$. From Table \ref{tab:addlabel11}, when the parameter $\lambda$ is too small, the method cannot achieve promising performance, since the gradient-level loss only captures the structural information. When the parameter $\lambda$ increases monotonically, the performance of DRN increases gradually. This demonstrates that the combination of the gradient-level loss and the pixel-level loss is effective to achieve promising results. In our setting, we empirically set $\lambda = 2$, since we find that a larger value usually does not bring further performance improvement.

\begin{table}[htbp]
	\centering
	\ygyan{
		\caption{Performance w.r.t. different values of $\lambda$.}
		\begin{tabular}{c||c|c|c|c|c|c}
			\hline
			$\lambda$ & 0.1 & 0.5 & 1.0 & 1.5 & 2.0 & 5.0 \\
			\hline
			PSNR & 30.44 & 31.55 & 31.92 & 32.13 & 32.24 & 32.24 \\
			\hline
		\end{tabular}%
		\label{tab:addlabel11}%
	}
\end{table}%

%

\subsection{Effect of training data ratio}

We conduct an experiment on SET5 to evaluate the influence of the number of training data. From Figure~\ref{fig:psnr_vs_dataRatio1}, when increasing the ratio of training data on the whole dataset, the values of PSNR score increases gradually. In addition, DRN consistently outperforms DBPN on all the data ratios.

\begin{figure}[htbp]
	\centering
	\includegraphics[width = 0.47\columnwidth]{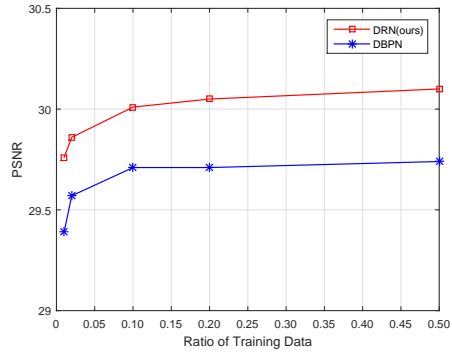}
	\caption{The results in different magnitude of ImageNet dataset.}
	\label{fig:psnr_vs_dataRatio1}
\end{figure}

\subsection{Comparison of model complexity}

We report the PSNR scores and the numbers of the parameters in DRN and several state-of-the-art models on $4\times$ and $8\times$ SR in Figures \ref{fig:performance_vs_parameters_4x} and~\ref{fig:performance_vs_parameters_8x}, respectively. The x-axis represents the number of the model parameters, and the y-axis means the value of PSNR. The results show that the proposed DRN method can achieve the best performance on both two datasets with the lowest computational complexity compared with the other baseline methods.

\begin{figure}[htbp]
	\centering
	\includegraphics[width = 0.7\columnwidth]{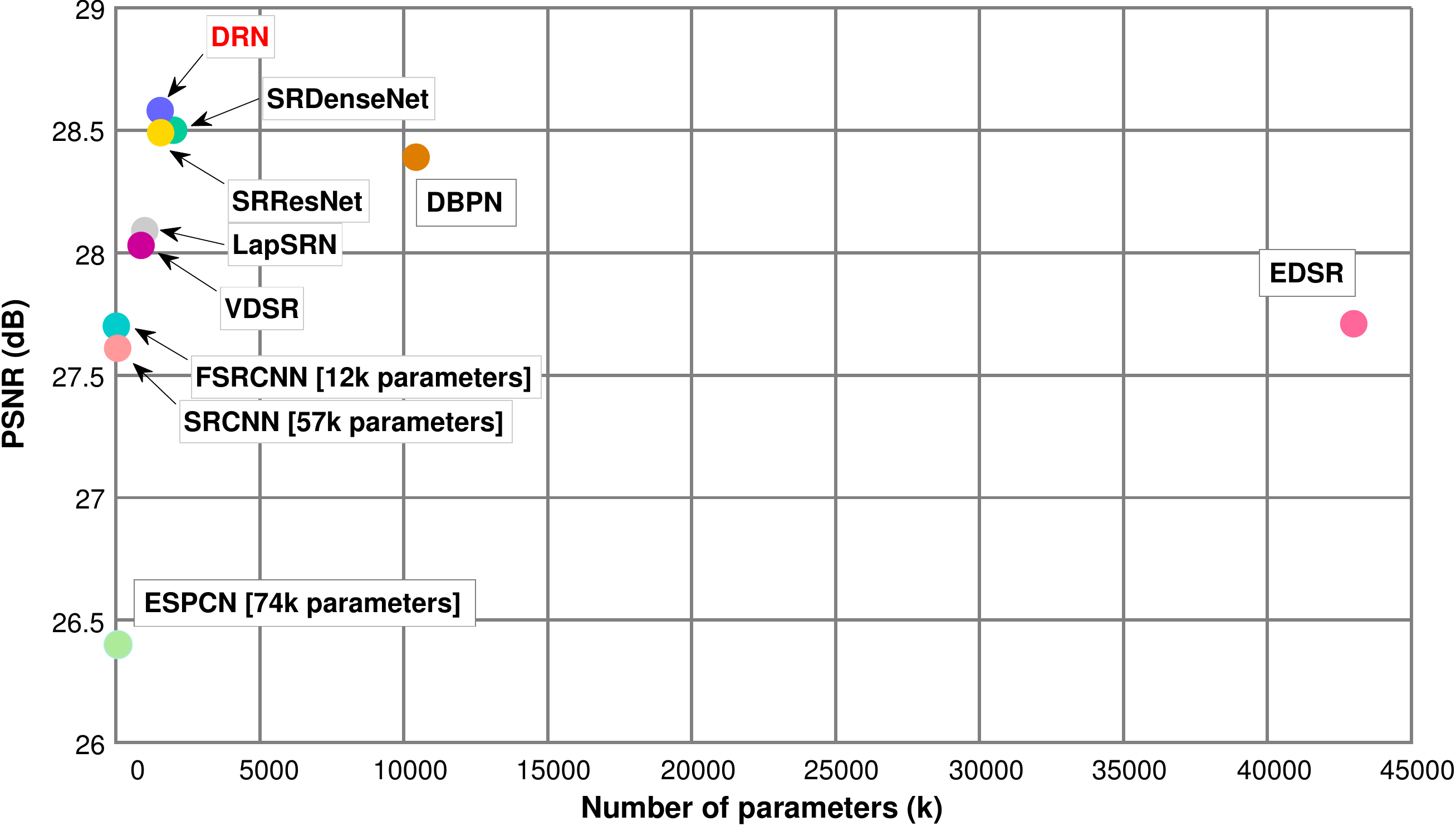}
	\caption{\qi{The results for $4\times$ SR on SET14 dataset.}}
	\label{fig:performance_vs_parameters_4x}
\end{figure}

\begin{figure}[htbp]
	\centering
	\includegraphics[width = 0.7\columnwidth]{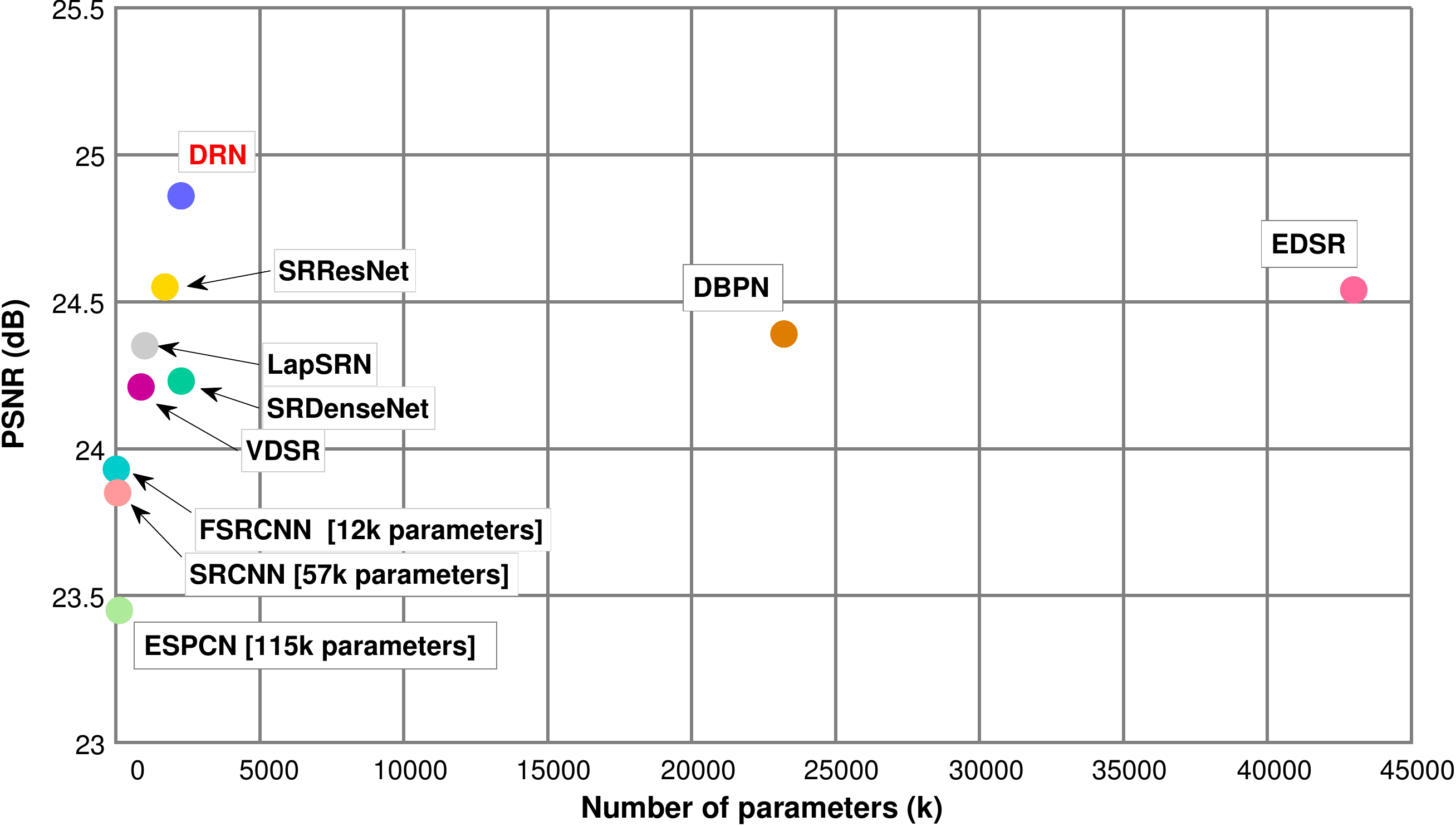}
	\caption{\qi{The results for $8\times$ SR on SET14 dataset.}}
	\label{fig:performance_vs_parameters_8x}
\end{figure}

\subsection{More results}
\qi{For further comparison, we provide the experimental results compared with the baseline methods for $8\times$ SR on several benchmark datasets.}

\begin{figure}[htbp]
	\centering
	\includegraphics[width = 0.9\columnwidth]{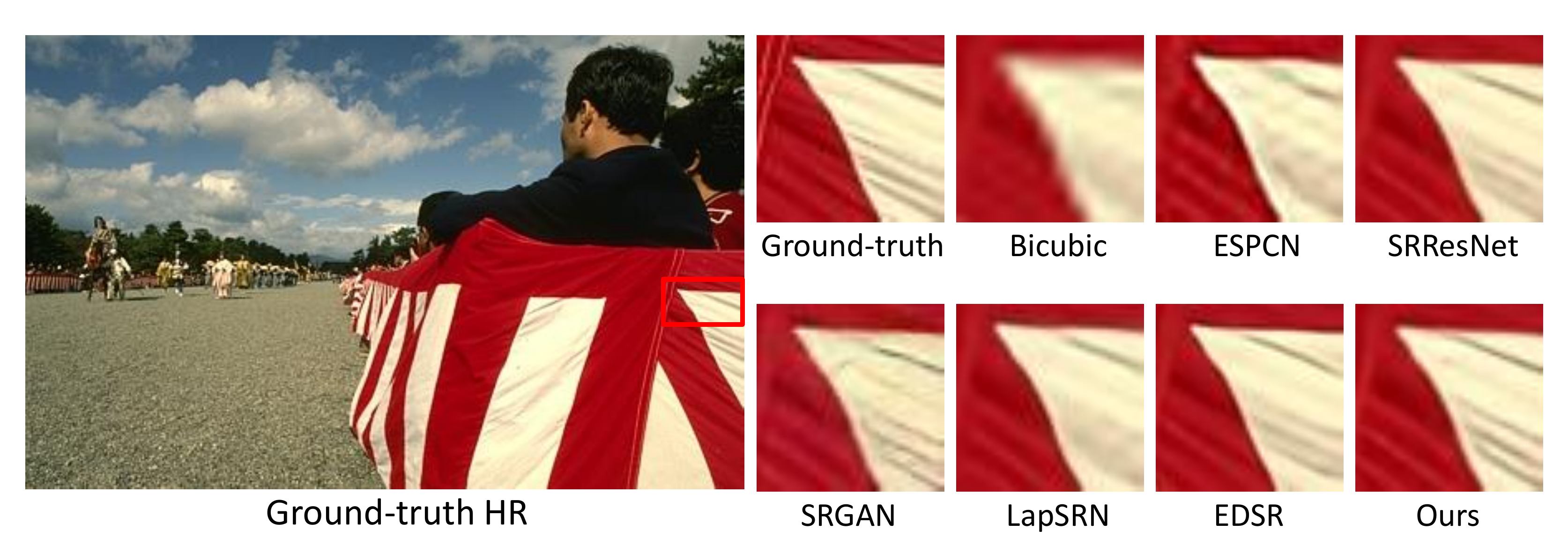}
	\includegraphics[width = 0.9\columnwidth]{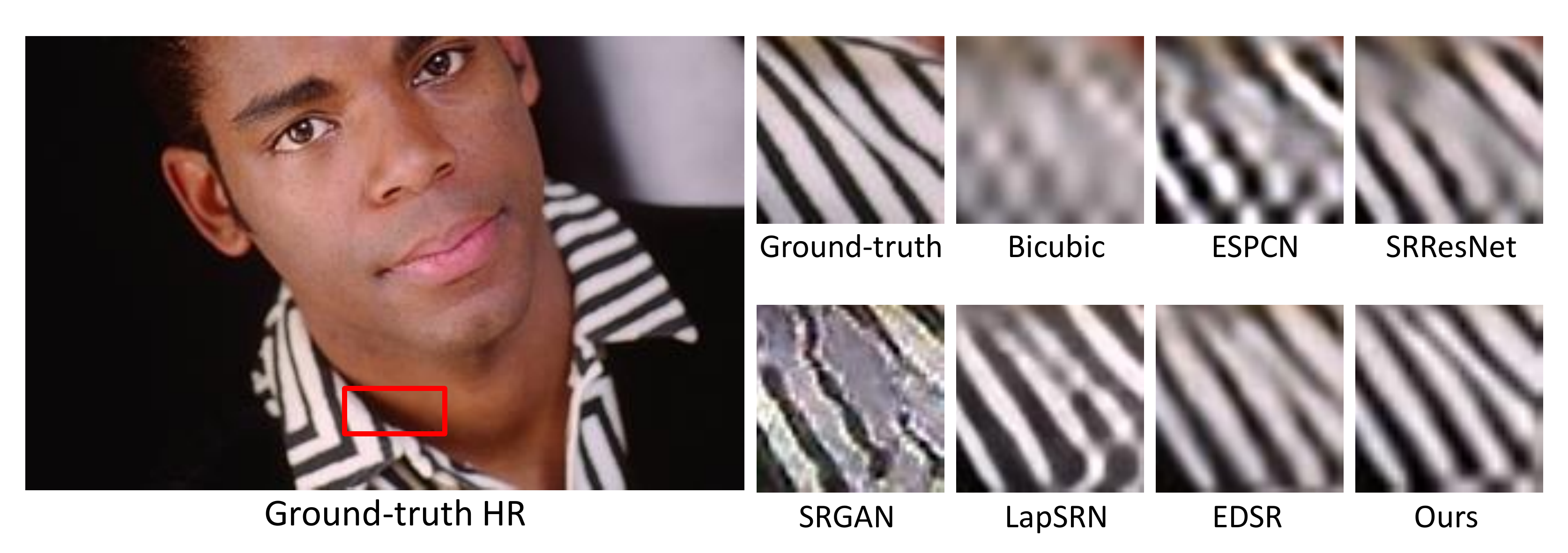}
	\includegraphics[width = 0.9\columnwidth]{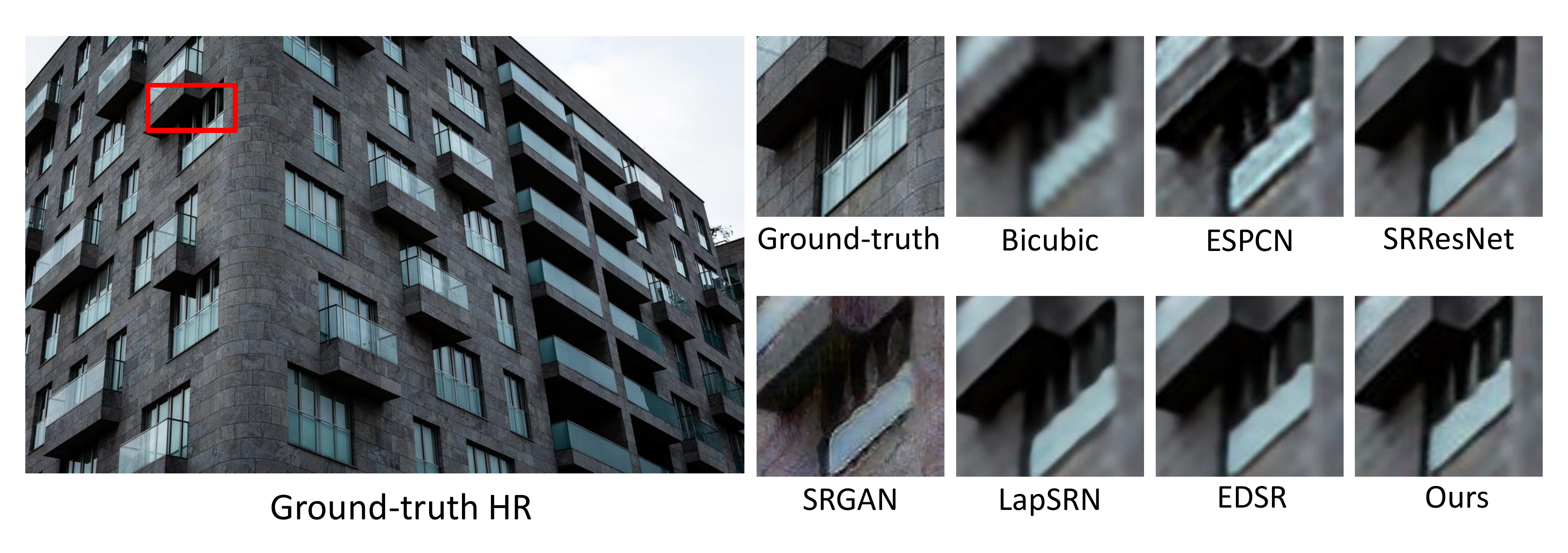}
	\includegraphics[width = 0.9\columnwidth]{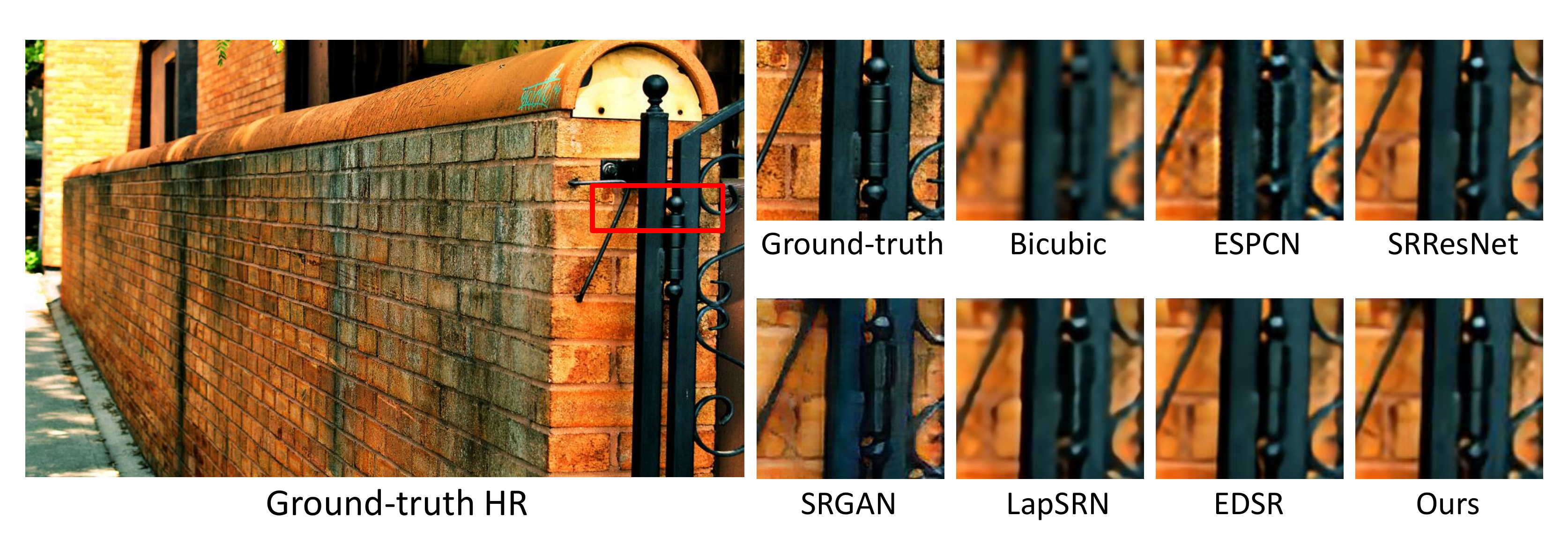}
	\caption{\qi{More results of visual comparison for $8\times$ upscaling super-resolution.}}
	\label{fig:image1}
\end{figure}

\end{document}